%% file: emnlp2023.tex
\pdfoutput=1

\documentclass[11pt]{article}

\usepackage[]{EMNLP2023}

\usepackage{times}
\usepackage{latexsym}
\usepackage{amsmath, amssymb, amsfonts, amsthm}
\usepackage{booktabs} 
\newtheorem{theorem}{Theorem}
\usepackage{tcolorbox}
\usepackage{multirow}
\usepackage{xcolor}

\usepackage[T1]{fontenc}

\usepackage[utf8]{inputenc}

\usepackage{microtype}

\usepackage{inconsolata}

\newtheorem*{theorem*}{Theorem}
\title{An Empirical Comparison of Vocabulary Expansion \\and Initialization Approaches for Language Models}

\author{Nandini Mundra$^{1,*}$ \quad Aditya Nanda Kishore$^{1,*}$ \quad Raj Dabre$^{1,3,6}$ \\ 
    {\bf Ratish Puduppully}$^{4,\dagger}$ \quad {\bf Anoop Kunchukuttan}$^{1,5}$ \quad {\bf Mitesh M. Khapra}$^{1,2}$ \\
    $^{1}$Indian Institute of Technology Madras \quad $^{2}$Nilekani Centre at AI4Bharat \\
    $^{3}$National Institute of Information and Communications Technology, Japan \\ 
    $^{4}$IT University of Copenhagen, Denmark \quad $^{5}$Microsoft India \\
    $^{6}$Indian Institute of Technology Bombay\\
    \textbf{Correspondence:} \texttt{miteshk@cse.iitm.ac.in, raj.dabre@nict.go.jp}}

\begin{document}
\maketitle
\def\thefootnote{*}\footnotetext{Equal contribution.}
\def\thefootnote{$\dagger$}\footnotetext{Work done while the author was at A$^{*}$STAR, Singapore.}
\def\thefootnote{\arabic{footnote}}

\begin{abstract}
Language Models (LMs) excel in natural language processing tasks for English but show reduced performance in most other languages. This problem is commonly tackled by continually pre-training and fine-tuning these models for said languages. A significant issue in this process is the limited vocabulary coverage in the original model’s tokenizer, leading to inadequate representation of new languages and necessitating an expansion of the tokenizer. The initialization of the embeddings corresponding to new vocabulary items presents a further challenge. Current strategies require cross-lingual embeddings and lack a solid theoretical foundation as well as comparisons with strong baselines. In this paper, we first establish theoretically that initializing within the convex hull of existing embeddings is a good initialization, followed by a novel but simple approach, \textit{Constrained Word2Vec (CW2V)}, which does not require cross-lingual embeddings. Our study evaluates different initialization methods for expanding RoBERTa and LLaMA 2 across four languages and five tasks.  The results show that CW2V performs equally well or even better than more advanced techniques. Additionally, simpler approaches like multivariate initialization perform on par with these advanced methods indicating that efficient large-scale multilingual continued pretraining can be achieved even with simpler initialization methods. We release our code publicly.\footnote{\url{https://github.com/AI4Bharat/VocabAdaptation_LLM/tree/CW2V}}

\end{abstract}

\section{Introduction}

Language models are adept at a wide spectrum of natural language processing (NLP) tasks \citep{liu2021pretrain, chung2022scaling, JMLR:v24:22-1144, 10.5555/3600270.3602070, goyal2023news, touvron2023llama}. However, the best-performing language models work well for English but have inferior capabilities in other languages. A common method to improve the capabilities of other languages is to continually pre-train and finetune the English model for other languages \cite{NEURIPS2019_c04c19c2}. This approach builds upon the capabilities acquired through large-scale English pre-training and focuses on aligning the English and other language spaces, making efficient re-use of compute and data resources \citep{cahyawijaya2023instructalign, zhang2023bayling}. One of the major challenges for LLM adaptation is the lack of vocabulary coverage in the original model's tokenizer for the new language. This would mean the inability to represent the new language if the vocabulary is totally different or inefficient tokenization with high fertility in the case of inadequate vocabulary representation.

A solution is to expand the tokenizer to incorporate new vocabulary and then perform continual pre-training on monolingual data from the new language to adapt the model to the new language \citep{cui2023efficient, nguyen2023seallms, minixhofer-etal-2022-wechsel}. In this scenario, an important question is: \textit{How do we initialize the embeddings of the new vocabulary items?} Various methods have been proposed in the literature for the initialization of the new token embeddings, from simple random initialization \citep{antoun-etal-2020-arabert, martin-etal-2020-camembert} to the mean of embeddings \citep{gee-etal-2022-fast} to sophisticated methods such as OFA among others \citep{minixhofer-etal-2022-wechsel, Dobler_2023, tran2020english, liu2024ofa} that learn the new embeddings as a function of existing embeddings using external resources and tools like cross-lingual word-vectors and bilingual dictionaries. However, there is no theoretical basis for what constitutes a \textit{good initialization}. Furthermore, in existing works, comparisons with simple, naive initialization methods across different model sizes are missing.

In this paper, we theoretically define and analyze the properties of a \textit{good initialization}. We prove that initializing embeddings of new vocabulary embeddings to be in the convex hull of original embeddings ensures that the greedy generation of the existing language(s) is not impacted by the new vocabulary items on initialization. 
Based on these insights, we propose a simple learnable initialization approach which we dub as \textit{Constrained Word2Vec (CW2V)} which ensures initializations in the convex hull without needing cross-lingual embeddings. We conducted a comparative analysis of CW2V alongside 5 existing initialization strategies including OFA on two models containing varying  parameters, namely RoBERTa (125M) and LLaMa2 (7B), examining their impact through 5 downstream tasks across 4 languages. Our analysis of various initialization methods demonstrates that CW2V achieves better if not comparable performance with the previous best methods. Additionally, we find that simpler methods like multivariate or mean initialization, which ensure new embeddings remain within the convex hull, are comparable with more advanced approaches such as OFA.

\section{Related Work}

\noindent \textbf{Multilingual Models}: To create a multilingual model for specific languages, one method is to train the model from scratch on the target languages using MLM and CLM objectives \cite{workshop2023bloom, conneau-etal-2020-unsupervised}. However, this requires significant computational resources and data. A more efficient approach is to adapt an existing pre-trained language model (PLM) \cite{devlin2019bert, touvron2023llama, MosaicML2023Introducing} to the desired target language. There are two ways to adapt a PLM to a new language. The first is to fully adapt the model to the new language, replacing the source tokenizer and focusing only on the new language’s performance \cite{minixhofer-etal-2022-wechsel, artetxe-etal-2020-cross}. The second is to keep the original language support and add the new language, ensuring the model still performs well on the source language \cite{garcia-etal-2021-towards, liu2024ofa}. In this work, we focus on extending the language support of the PLM rather than replacing it. We do this by extending the source tokenizer, which requires effectively initializing the model's embedding layer and LM head for the added tokens in the vocabulary.

\noindent \textbf{Embedding Initialization Strategies}: Previous work has focused on different initialization strategies. Methods like FVT \cite{gee-etal-2022-fast} and \citet{hewitt2021initializing} use the mean of source PLM embeddings, while WECHSEL \cite{minixhofer-etal-2022-wechsel}, RAMEN \cite{tran2020english}, FOCUS \cite{Dobler_2023}, and OFA \cite{liu2024ofa} utilize external cross-lingual word vectors and source embeddings. However, these approaches rely on static embeddings. In contrast, we propose initialization strategy that learns new embeddings from the source PLM model and doesn't require static embeddings.

\noindent \textbf{Continual Pre-training}: A good initialization strategy provides a solid start for adapting a PLM to a new language by effectively initializing the new tokens in the embedding and LM head layers. However, to fully adapt the extended model to the new language, continued pre-training (CPT) \cite{wang2022multimodal, alabi2022, imani2023} is essential. Therefore, we performed CPT on target languages post initialization.



\section{Methodology}
We describe the core methodology in this work followed by theoretical proofs of \textit{good initializations} which motivate our own initialization approach, namely, Constrained Word2Vec.

\begin{figure}[t]
    \centering
    \includegraphics[scale = 0.28]{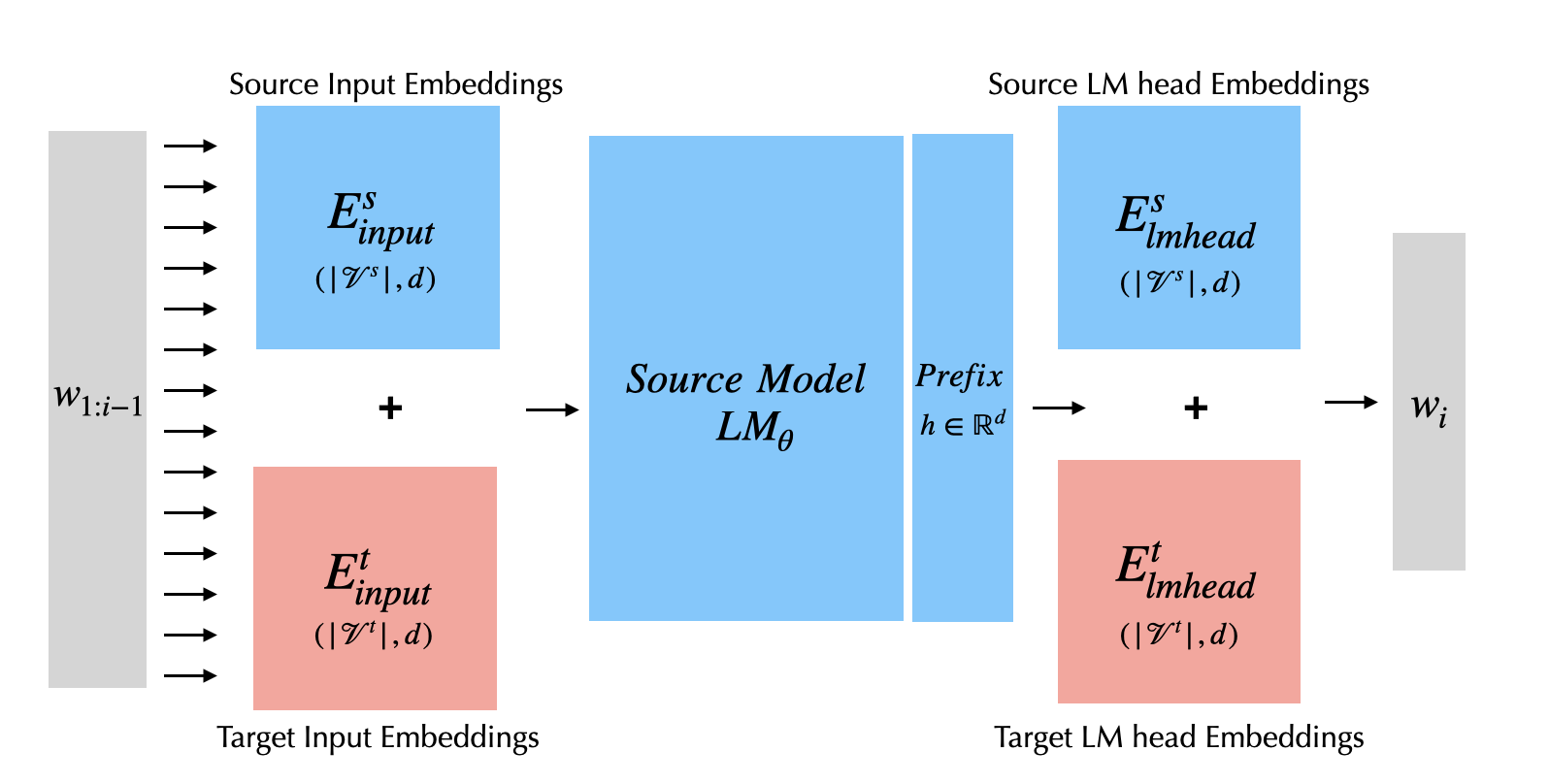}
    \caption{Setup for Vocabulary Expansion. Source model is shown in blue blocks, and expanded vocabulary embeddings are represented in red blocks. Source model parameters remain unchanged.}
    \label{fig:enter-label}
\end{figure}

\subsection{Vocabulary Expansion}
We adapt the same vocabulary expansion problem formulation as \citet{hewitt2021initializing}. Let $\theta$ be the parameters of a pre-trained neural source language model $LM_{\theta}^s$, and let $\mathcal{V}^s=\{v_1^s, v_2^s, \ldots, v_n^s\}$ be the vocabulary of $LM_{\theta}^s$. We will refer to $\mathcal{V}^s$ as the source vocabulary henceforth. Let $e_i^s \in \mathbb{R}^d$ be the sub-word embedding for word $i \in \mathcal{V}^s$. Let $E^s$ denote the language modeling head's (henceforth \textit{LM head}) embedding matrix  of $LM_{\theta}^s$ and this is our source embedding matrix. 
The probability of occurrence of the next word $w_i$ given the previous word sequence $w_{1:i-1}$, $p_\theta\left(w_i \mid w_{1:i-1}\right)$, is given by
\[
p_\theta\left(w_i \mid w_{1:i-1}\right) = \frac{\exp\left(h_{i-1}^{\top} e_{w_i}^s\right)}{\sum_{j \in \mathcal{V}^s} \exp\left(h_{i-1}^{\top} e_j^s\right)},
\]
where the prefix $h_{i-1} = \phi\left(w_{1:i-1}; LM_\theta^s\right) \in \mathbb{R}^d$ is the neural representation of the input using $LM_\theta^s$. 

In vocabulary expansion, we add $n'$ new sub-words $\notin \mathcal{V}^s$ forming the target vocabulary $\mathcal{V}^t = \{v_1^t, v_2^t, \ldots, v_{n'}^t\}$. This implies we need a new word embedding $e_j^t$ for each $j \in \mathcal{V}^t$ comprising in $E^t$. The new language model $LM_{\theta'}^t$ has parameters $\theta' = \theta \cup \{e_j^t; j \in \mathcal{V}^t\}$. The output distribution of $LM_{\theta'}^t$ given by $p_{\theta'}\left(w_i \mid w_{1:i-1}\right)$  is defined similarly as $p_\theta\left(w_i \mid w_{1:i-1}\right)$ but with the normalization factor involving $\mathcal{V}^s \cup \mathcal{V}^t$.


Our goal is to find initializations for $E^t$ such that the extended model not only retains its previous behavior but also can lead to good downstream performance for the languages corresponding to the new vocabulary with minimal continual pre-training.  Retaining performance in English is particularly beneficial, as the knowledge embedded in English models often supports performance in other languages \cite{pires-etal-2019-multilingual}. Figure~\ref{fig:enter-label} gives an overview of our approach. Note that in our notations so far we have only mentioned the LM head, but just as the LM head has an expansion ($E^{t}_{lmhead}$), the input embedding matrix also has an expansion ($E^{t}_{input}$). This is trivial if both matrices are shared but in case they are not, we also need to find initializations for the latter. Following \citet{hewitt2021initializing}, we can use the same approach to initialize $E^{t}_{input}$ as we do for $E^{t}_{lmhead}$.




\subsection{What is a `good' embedding initialization?}

As we are ensuring that the model parameters $\theta$ remain unchanged at the initialization step, we can safely say that for the same word sequence $w_{1:i-1}$, where each word in the sequence belongs to $\mathcal{V}^s$, the prefix $h_{i-1}$ at the output layer remains the same. Thus, the output word $w_i$ strictly depends on the embeddings of the new words added to the vocabulary, as they determine the new partition function and the output probability distribution. The main goal of our analysis is to identify the set of initializations of new words that give us the same output before and after expansion for the prefixes formed by the original tokens. In other words, for the same input word sequence $w_{1:i-1}$, where $w_k \in \mathcal{V}^s \ \forall \ k \in [i-1]$, if $w_i \ \text{and} \ w_i'$ represent the words predicted by language models $LM_{\theta}^s$ and $LM_{\theta'}^t$ respectively, i.e., $w_i = \operatorname{argmax}_{j \in \mathcal{V}^s} \ p_\theta\left(j \mid w_{1: i-1}\right)$ and $w_i' = \operatorname{argmax}_{j \in \mathcal{V}^s \cup \mathcal{V}^t} \ p_{\theta'}\left(j \mid w_{1: i-1}\right)$, we need $w_i = w_i'$. Let $e_1^t, e_2^t, \hdots, e_{n'}^t \in \mathbb{R}^d$ be the embedding initializations for words in $\mathcal{V}^t$. Therefore, a \textit{good initialization} is an initialization $\left\{e^t_{j}; j \in \mathcal{V}^t \right\}$ that ensures, for any prefix $h_{i-1} \in \mathbb{R}^d$, the set of prefixes formed by word sequences from the source vocabulary, that is $'w_i = w_i'$.

\subsection{Theorems}

\begin{theorem}: \textbf{A good initialization preserves the pre-expansion behavior.}

Let $e_1^s, e_2^s, e_3 ^s, ..., e_n^s \in \mathbb{R}^d$ be the embeddings of words in $\mathcal{V}^s$. Let $e_1^t,e_2^t, \hdots , e_{n'}^t \in \mathbb{R}^d $ be the embedding initializations for words in $\mathcal{V}^t$. If
    \begin{align} \label{condition}
      \sup_{k \in \mathcal{V}^t} (h^Te_{k}^t) \leq \sup_{k \in \mathcal{V}^s}  (h^Te_k^s) 
    \end{align}
    holds for all h $\in$ $\mathbb{R}^d$, then $\left\{e^t_{j}; j \in \mathcal{V}^t \right\}$ is a `good' initialization. 
     \label{theorem1}
\end{theorem} 

\begin{proof}
    Let $h = h_{i-1} \in \mathbb{R}^d$ be a prefix formed by a word sequence $w_{1:i-1}$, where $w_k \in \mathcal{V}^s \ \forall \ k \in [i-1] \ $. As condition \ref{condition} holds for all $h \in \mathbb{R}^d$, we can say that, 
    \begin{align*}
        \sup_{k \in \mathcal{V}^t} (h^Te_{k}^t) \leq& \sup_{k \in \mathcal{V}^s}   (h^Te_k^s) \label{eq:1}\\
        \implies \sup_{k \in \mathcal{V}^t} \exp(h^Te_{k}^t) \leq & 
        \sup_{k \in \mathcal{V}^s}   \exp(h^Te_k^s) \\
        \implies \sup_{k \in \mathcal{V}^t} \frac{\exp(h^Te_{k}^t)}{Z'} &\leq
        \sup_{k \in \mathcal{V}^s}  \frac{ \exp(h^Te_k^s)}{Z'}
    \end{align*}

    where, $Z' = \sum_{j \in \mathcal{V}^s \cup \mathcal{V}^t} \exp \left(h^{\top} e_{j}^t\right)$

    is the new partition function, which is a positive constant as prefix and all the embeddings are given. We know that, $ \frac{\exp(h^Te_{k}^t)}{Z'}$ represents the probability of occurrence of word corresponding to the embedding $e_{k}^t$ at time step $i$. Thus, the inequality just says that probability of occurrence of any word from target vocabulary $\mathcal{V}^t$ is less than or equal to probability of occurrence of a word from source vocabulary. As the decoding at output layer is greedy, the output word is going to come from source vocabulary. We can guarantee that it remains the same as pre-expansion model's output word because the prefix remains the same before and after expansion as $w_k \in \ \text{source vocabulary} \ \mathcal{V}^s \ \forall \ k \in [i-1] \ $. Hence, as $w_i = w_i'$ and the embedding initialization $\left\{e^t_{j}; j \in \mathcal{V}^t \right\}$  is `good'. 
\end{proof}

\begin{theorem}: \textbf{An initialization in the convex hull of source embeddings is good.}   \label{theorem2}

If $y \in \mathcal{S}$, where $\mathcal{S}$ is the convex hull of the embeddings $e_1^s, e_2^s, e_3^s, \ldots, e_n^s$, then $(h^Ty) \leq \sup_{k \in \mathcal{V}^s} (h^Te_k^s)$ for all $h \in \mathbb{R}^d$. Moreover, if $e_i^t \in \mathcal{S}$ for all $i \in \mathcal{V}^t$, then the initialization is `good'.
\end{theorem}

\begin{proof}
    Given $y \in \mathcal{S}$. Thus $y$ can be written as $y = \sum_{j \in \mathcal{V}^s} \alpha_j e_j^s$ where $\sum_{j \in \mathcal{V}^s} \alpha_j = 1$ and $0 \leq w_j \leq 1 \ \forall  \ j \in \mathcal{V}^s $. Thus, $\forall \ h \in \mathbb{R}^d,$
    
    \[ h^Ty = \sum_{j \in \mathcal{V}^s} \alpha_j h^Te_j^s \].

    

As $0 \leq \alpha_j \leq 1 \ \forall  \ j \in \mathcal{V}^s $,

\[
(h^Ty) \leq \sup_{k \in \mathcal{V}^s} (h^Te_k^s)
\]
 Given $e_i^t \in \mathcal{S} \ \ \forall \ i \in \mathcal{V}^t$
    \begin{align*}
        \implies (h^Te_i^t) \leq \sup_{k \in \mathcal{V}^s}& (h^Te_k^s)   \ \ \forall \ i \in \mathcal{V}^t  \ \ \forall \ h \in \mathbb{R} ^d \\
        \implies  \sup_{k \in \mathcal{V}^t} (h^Te_{k}^t) &\leq \sup_{k \in \mathcal{V}^s}  (h^Te_k^s) \ \ \forall \ h \in \mathbb{R} ^d \\
    \end{align*}

Thus, from theorem \ref{theorem1} we can say that if $e_i^t \in \mathcal{S} \ \ \forall \ i \in \mathcal{V}^t$, then the initialization is \textit{good}.
\end{proof}

We have showed that as long as we initialize every target embedding vector as a weighted average of source embeddings, the model output remains the same for the same prefix as long as it is obtained from a word sequence formed only by source vocabulary, thereby making it \textit{good}.  Table \ref{fig:model_outputs} verifies this empirically. 
In Appendix~\ref{converse} we provide some additional theoretical analysis where we show a weaker converse of Theorem \ref{theorem2}, that any \textit{strongly good} initialization lies in the convex hull of source embeddings.





\subsection{Our Approach: Constrained Word2Vec}
Having established that a initializing in the convex hull of existing embeddings is \textit{good}, we now propose \textit{Constrained Word2Vec (CW2V)}, a novel approach to learn these initializations. Specifically, we constrain $E^t$ as $WE^s$ where $\sum_{j \in \mathcal{V}^s} W_{ij} = 1 \  \forall \  i \in \mathcal{V}^t$ and $W_{ij} \geq 0 \ \forall \  j \in \mathcal{V}^s ,  i \in \mathcal{V}^t$. Here, $E^s \in \mathbb{R}^{(|\mathcal{V}^s|, d)}$ is the source embedding matrix,  $E^t \in \mathbb{R}^{(|\mathcal{V}^t|, d)}$ is the target embedding matrix and $W \in \mathbb{R}^{(|\mathcal{V}^t|, |\mathcal{V}^s|)}$ is the weight matrix that transforms $E^s$ to $E^t$ while ensuring the target embedding vectors reside inside the convex hull of the source embedding vectors. Our goal is to learn $W$.

Let $\mathcal{E}^t$ be the post-expansion embedding matrix of size $(|\mathcal{V}^s \cup \mathcal{V}^t|,d)$. In other words, $\mathcal{E}^t = [E^s; WE^s]$ where $;$ indicates concatenation along the vocabulary axis. By using $\mathcal{E}^t$ as the embedding matrix with $W$ as the only learnable parameters, we propose a mechanism similar to Skip-gram \cite{mikolov2013distributed} to obtain $\mathcal{E}^t$. In many modern PLMs, such as LLaMA, the input and output embedding layers are not tied, necessitating separate weight matrices for the input embedding and the LM head defined as $\mathcal{E}_{input}^s$ = [$E_{input}^s; softmax(W_{input})E_{input}^s$], $\mathcal{E}_{LM-head}^s$ = [$E_{LM-head}^s; softmax(W_{LM-head})E_{LM-head}^s$]$^T$ with sizes $(|\mathcal{V}^s \cup \mathcal{V}^t|, d)$, $( d, |\mathcal{V}^s \cup \mathcal{V}^t|)$, respectively. The \textit{softmax} operation ensures that the weights in each row add up to 1, thus assuring that the target embedding vectors remain in the convex hull of pre-expansion embeddings. 




We set these embedding matrices $\mathcal{E}^t_{input}$ and $\mathcal{E}^t_{LM-head}$ up in the traditional Skip-gram architecture \cite{mikolov2013distributed} as the word and context representation matrices. Similar to OFA \cite{liu2024ofa}, in order to make the learning computationally more efficient, we can also factorise $W_{input}$ and $W_{LM-head}$ and learn the resulting parameters. This methodology can be extended to any PLM. If both the embedding layers are tied for a PLM (RoBERTa), we still learn two weight matrices and choose either for initializing $E^t$.  To align target language embeddings with English, we trained the CW2V model on monolingual data from all languages and bilingual English-to-target dictionaries.

\section{Experimental Setting}
We now describe the models we focus on, the languages, downstream tasks and datasets, and implementation details.


\subsection{Models}
We use RoBERTa  \cite{DBLP:journals/corr/abs-1907-11692}, an encoder based architecture and LLaMA2-7B \cite{touvron2023llama}, a decoder based model and employ these models as the source models for our multilingual vocabulary expansion experiments.

\subsection{Tokenizers} \label{tokenizer}

We use the RoBERTa tokenizer as the source tokenizer for experiments with RoBERTa and the LLaMA2 tokenizer as the source tokenizer for experiments with LLaMA2. Since we are focusing on multilingual transfer, we train a SentencePiece \cite{DBLP:journals/corr/abs-1808-06226} tokenizer using textual data in target languages (German, Russian, Hindi, Tamil) and merge the obtained tokenizer with LLaMA2's tokenizer. The resulting tokenizer has 57K subwords in its vocabulary, and this merged tokenizer serves as the unified target tokenizer for all of our experiments, even for experiments with RoBERTa. We identify common subwords using a `fuzzy' search similar to FOCUS \cite{Dobler_2023} and OFA \cite{liu2024ofa}. We report the fertility score of the target tokenizer in all four target languages in Appendix \ref{sec:a-fertility}. Vocabulary expansion significantly reduces the fertilities for the languages considered.

\subsection{Datasets and Languages}  \label{data}
We extended the source model (English) to four target languages: Hindi, Tamil, Russian, and German. For all training, the Hindi and Tamil datasets were sourced from SANGRAHA \cite{khan2024indicllmsuite}, while the Russian, German, and English datasets were sourced from OSCAR \cite{ortiz-suarez-etal-2020-monolingual}. To train the multilingual tokenizer, we used a monolingual dataset of 3 million sentences per target language, sourced from the tokenizer training data used in IndicTrans2 \cite{gala2023indictrans2highqualityaccessiblemachine}. For the constrained word2vec model training, we used a monolingual corpus of 2 million tokens per target language. Additionally, we incorporated bilingual dictionary datasets: Hindi and Tamil from \cite{kanojia-etal-2018-indian}, German from url \footnote{\url{https://nlp.stanford.edu/projects/nmt/data/wmt14.en-de/dict.en-de}} processed by \cite{bojar-EtAl:2014:W14-33}, and Russian from url \footnote{\url{https://github.com/Badestrand/russian-dictionary}}.
Each expanded and initialized model underwent further pre-training on a multilingual dataset of 2.5 billion tokens, combining 500 million tokens per target language and 500 million English tokens.

\subsection{Baselines} \label{baselines}


\noindent\textbf{OFA} The One For All (OFA) Framework (Liu et al., 2024) \cite{liu2024ofa} uses multilingual static word vectors to inject alignment knowledge into the new subword embeddings. Regardless of the factorization approach, OFA initializes all new target embeddings using a weighted average of the source vocabulary embeddings, making OFA a `strongly good' initialization.

\noindent\textbf{Univariate} Each target embedding is initialised by drawing values from 1-D Gaussian distributions parameterized by the mean and standard deviation of the source embeddings for each dimension. This was the primary baseline considered by OFA \cite{liu2024ofa}. 

\noindent\textbf{Multivariate} Every target embedding is sampled from the multivariate gaussian distribution of embeddings whose mean and covariance come from the original embeddings $E^s$.

\noindent\textbf{Mean} Every target embedding is the average of pre-expansion embeddings. Mean initialization is used to initialize target vocabulary in FVT \cite{gee-etal-2022-fast} and \citet{hewitt2021initializing}. This is a `strongly' good initialization as mean of original embeddings belongs to the convex hull of original embeddings.

\noindent\textbf{Random} Every target embedding is randomly sampled from the $d-$dimensional guassian distribution $\mathcal{N}(0,0.02I)$ where $I$ is a $d-$dimensional identity matrix.

\subsection{Constrained Word2Vec Training}

\input{table_skipgram_factorization}
We trained the constrained word2vec model using a similar setup to skip-gram \cite{mikolov2013distributed} training. The context window size was set to 10 and negative sampling to 5. Additionally, we factorized the $W_{input}$ and $W_{LM-head}$ matrices, with a factorized dimension of 1024. This factorization was done to reduce the number of trainable parameters, similar to OFA (\cite{liu2024ofa}). 
Factorizing the weight matrices in the constrained word2vec model for RoBERTa reduced the number of trainable parameters from 758M to 59M, and for LLaMA2, it reduced from 1660M to 118M.

\input{table_task}
\subsection{Downstream Tasks} \label{downstream}

\input{table_llama_initialized}

We evaluated RoBERTa and LLaMA on various tasks, as shown in Table~\ref{table:tasks}. For XNLI, we used XNLI \cite{conneau-etal-2018-xnli} for German, Russian, Hindi, and English, and IndicXNLI \cite{aggarwal-etal-2022-indicxnli} for Tamil. For NER, we used WikiANN \cite{pan-etal-2017-cross}. For QA, we used SQuAD \cite{rajpurkar2018know} for German, Russian, Hindi, and English, and IndicQA \cite{doddapaneni-etal-2023-towards} for Tamil. For Machine Translation, we used FLORES \cite{nllbteam2022language}. RoBERTa MLM checkpoints were fine-tuned on English and evaluated zero-shot on target languages. LLaMA CLM checkpoints were evaluated with 4-shot prompting. The metrics for each task are also listed in Table~\ref{table:tasks}.

\section{Results}

We now describe the results of our investigation, where we first evaluate different initialization methods without continual pre-training or fine-tuning for RoBERTa and LLaMA2. We follow this up with results for continual pre-training and fine-tuning for RoBERTa, and continual pre-training and few-shot prompting for LLaMA2.

\subsection{Impact of Initialization Methods}
\noindent\textbf{For the encoder-only RoBERTa model:} Table \ref{tab:results} presents the performance of the expanded RoBERTa model initialized with Constrained Word2Vec, alongside baseline models, across three downstream tasks: XNLI, NER, and QA. The expanded and initialized model was not continually pre-trained but was fine-tuned till convergence on downstream task data. Firstly, looking at the columns labeled \textbf{en}, we can see that CW2V is better than any baseline for English, even OFA, indicating that it preserves the pre-expansion behavior of RoBERTa better than any other methods. Next, the scores under the \textbf{avg} columns indicate that CW2V is competitive with other approaches, especially OFA but tends to be slightly inferior. This means that CW2V mildly sacrifices the performance on other languages while strongly preserving the English performance.

\noindent\textbf{For the decoder-only LLaMA2 model:} Table \ref{tab:results} shows the performance of the expanded LLaMA2 model initialized with Constrained Word2Vec, alongside baselines, on the following downstream tasks: XNLI, Machine Translation, QA and XLSUM (summarization). Here as well, the expanded and initialized model was not continually pre-trained but was evaluated using few-shot prompting. Different from the case of RoBERTa, the CW2V model significantly outperforms the OFA model across all tasks and languages despite not being continually pre-trained. CW2V achieves higher CHRF scores, averaged over all translation directions, in MT (17.02 En-X and 27.26 X-En) compared to OFA's 11.17 and 16.17, respectively. Similarly, for XNLI, QA and XLSUM, we observe that the average (\textbf{avg} column) performance over all languages for CW2V is vastly better than any other approach. The English-only performance (\textbf{en} column) however is comparable across all approaches with CW2V being only slightly better. This proves that in decoder-only models while CW2V is as good as any other approach for preserving the pre-expansion English-only performance, it is substantially better than other approaches for the new languages via vocabulary expansion. 

\begin{figure*}[!h]
    \centering
    \includegraphics[scale=0.35]{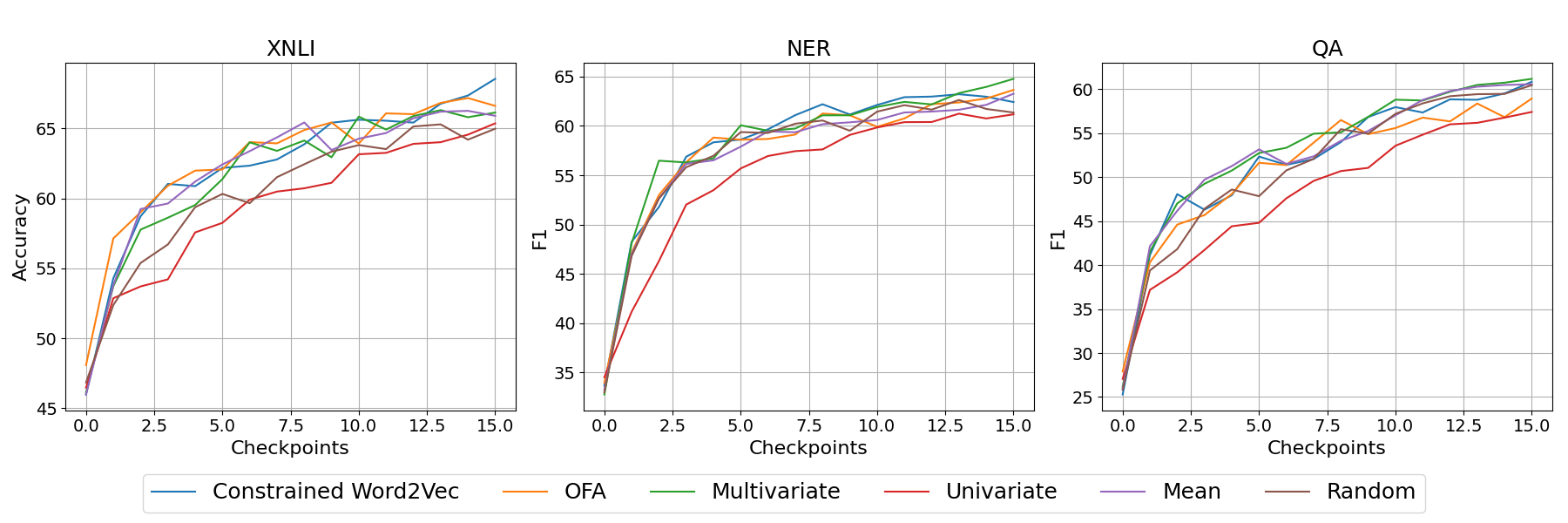}
    \caption{Evaluation of different initialization methods on expanded RoBERTa models using three multilingual tasks (XNLI, NER, QA) at 15 CPT checkpoints. The plots show average performance across five languages.}
    \label{fig_roberta_CPT}
\end{figure*}

\begin{figure*}[!h]
    \centering
    \includegraphics[scale=0.36]{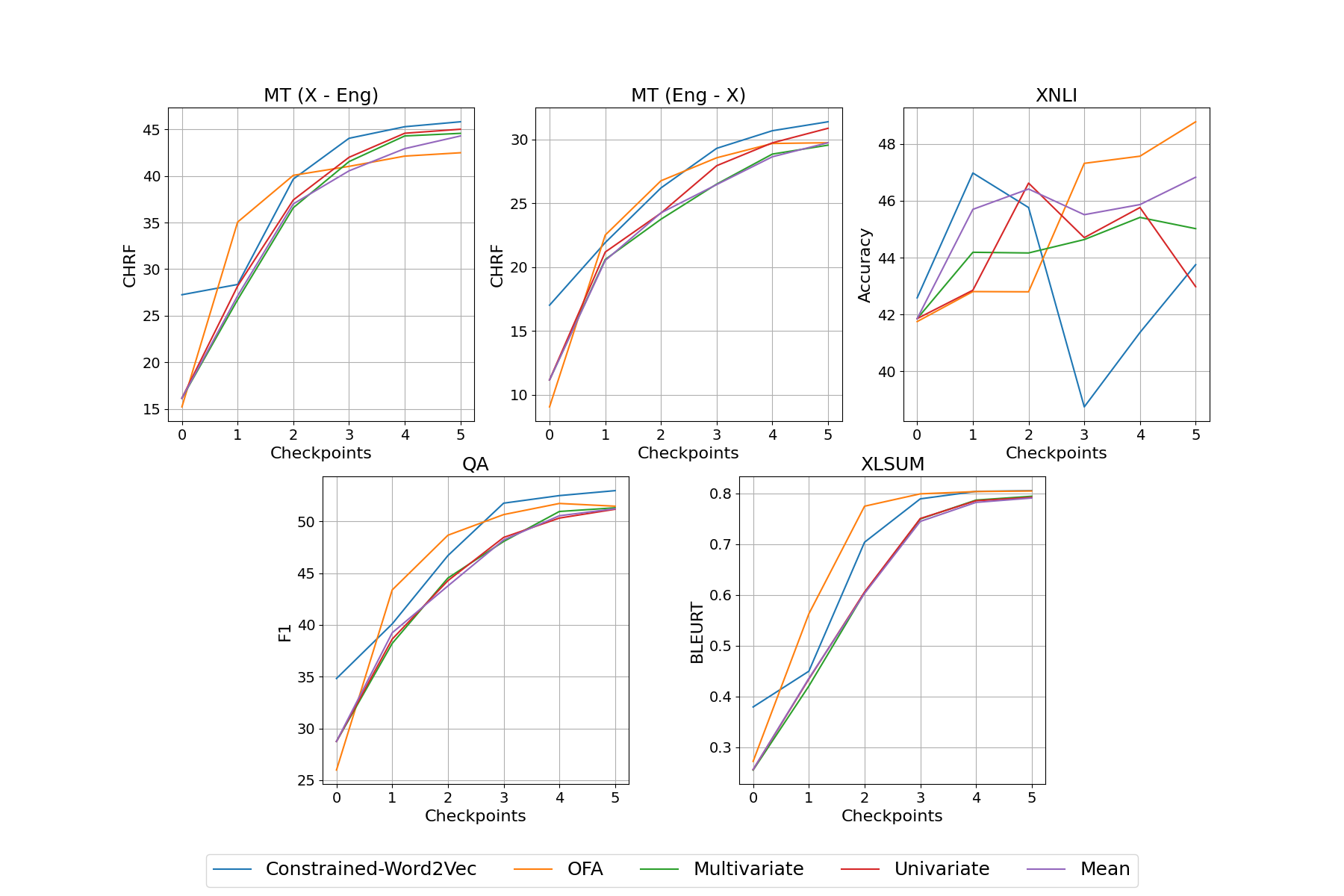}
    \caption{4-shot XNLI, MT, QA, XLSUM evaluation of different initialization methods on expanded LLaMA2 models at 5 equidistant CPT checkpoints. MT plots show average performance across 4 languages, and XNLI, QA, XLSUM plots show average performance across 5 languages.}
    \label{fig_llama_avg}
\end{figure*}

\begin{figure*}[!h]
    \centering
    \includegraphics[scale=0.35]{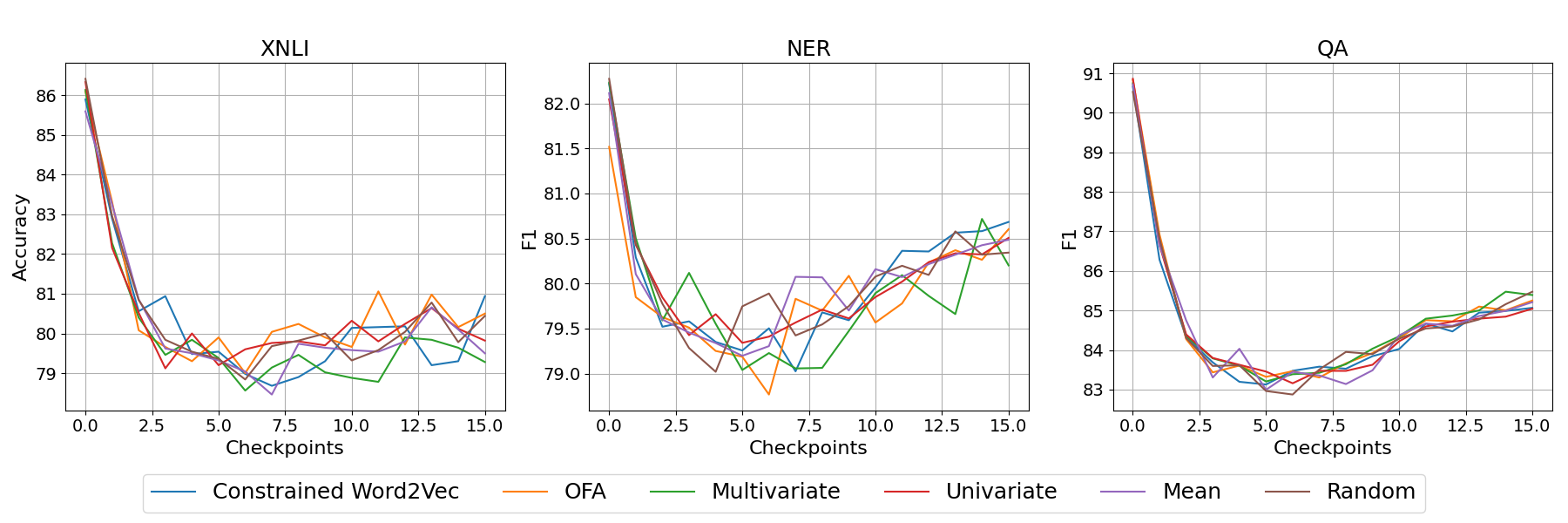}
    \caption{Assessment of English performance for various initialization methods on expanded RoBERTa models across three downstream tasks (XNLI, NER, QA) at 15 CPT checkpoints.} 
    \label{fig_roberta_en}
\end{figure*}

\subsection{Impact of Continual Pretraining}
Here we show the compounding effects of continual pre-training  and various initialization strategies to understand whether initialization matters or not when monolingual adaptation data exists.

\noindent\textbf{For the encoder-only RoBERTa model:} We evaluate the performance of expanded RoBERTa models initialized with Constrained Word2Vec (CW2V) and other baseline methods with CPT. We evaluate 15 checkpoints from one epoch of CPT (plus the initial checkpoint prior to CPT) on 3 downstream tasks. The results are depicted in Figure \ref{fig_roberta_CPT}. Here, again, CW2V demonstrates comparable or superior performance to OFA, especially towards the latter stages of CPT. As illustrated in Figure \ref{fig_roberta_CPT}, CW2V quickly converges with OFA (within less than 4 checkpoints) across all three tasks. Additionally, simpler baselines such as mean and multivariate also achieve comparable performance to OFA and CW2V shortly thereafter (in NER and QA, Multivariate catches up to CW2V within two checkpoints), demonstrating strong performance. This suggests that straightforward baselines like multivariate can be as effective as sophisticated methods such as Constrained Word2Vec and OFA. Furthermore, our analysis consistently shows that Univariate and Random initialization methods underperform in comparison to CW2V, OFA, Multivariate\footnote{Multivariate initialization has a high probability of residing within the convex hull of the source embeddings (Appendix \ref{sec:a-multi})}, and Mean. This highlights that Univariate and Random methods, despite being used as primary baselines in previous work, are inadequate for comparison. 

\noindent\textbf{For the decoder-only LLaMA2 model:} Similarly, we observe the performance of the expanded LLaMA2 models initialized with Constrained Word2Vec and the baselines. We evaluate 5 checkpoints from one epoch of CPT (plus the initial checkpoint prior to CPT) on 4 downstream tasks. The results are depicted in Figure \ref{fig_llama_avg}. For MT and QA, both generative tasks, on average, CW2V is better if not comparable with OFA while being consistently better than all other approaches. We see that CW2V quickly surpasses OFA in 2-3 checkpoints. In the case of XLSUM, however, OFA tends to be better during intermediate checkpoints (1, 2, 3), but CW2V eventually performs just as well afterwards. Once again, CW2V (and OFA) are significantly better than other baselines. 

XNLI is the only confounding task since no clear trends can be observed over various CPT stages. Furthermore, all models perform almost equally poorly, indicating that neither vocabulary expansion nor CPT is sufficient to improve XNLI performance. We suppose that fine-tuning on an XNLI dataset may shed further light on this, but due to limited compute, we did not pursue fine-tuning for any task and hence leave it as future work. Overall, CW2V is a highly effective initialization strategy for CPT, particularly benefiting languages that we aim to support more effectively through vocabulary expansion.

\subsection{Catastrophic Forgetting in English tasks}
Here we reveal something concerning about the inevitable negative effect of CPT on the pre-expansion language (English). During continued pre-training on monolingual datasets in both target and source languages, even with the source language (English) constituting 20\% of the total dataset, we observed an initial drop in English performance. Figure \ref{fig_roberta_en} shows the performance of the expanded RoBERTa models at various CPT checkpoints on only English tasks. Initially, performance drops, after which it begins to improve with prolonged training without comprising performance on non-english tasks. This suggests that adjusting the model to learn new target language data temporarily disrupts the weights previously optimized for English but prolonged training could potentially further restore and enhance English performance.

\section{Conclusion}
In this work, we establish that effective embedding initialization for an expanded vocabulary in language models can be achieved within the convex hull of source vocabulary embeddings. We introduce a data-driven initialization method, \textit{Constrained Word2Vec (CW2V)}, which learns the target embeddings by constraining them in the convex hull of the source embeddings. Our comparison of various initialization methods reveals that Constrained Word2Vec performs on par with other advanced techniques. Additionally, we find that simple methods like Multivariate and Mean, which ensure new embeddings lie within the convex hull of source embeddings, perform comparably well to more complex approaches. This indicates that efficient large-scale multilingual continued pretraining can be possible even with simpler methods, provided they are \textit{good} initialization strategies.


\bibliography{anthology,custom}
\bibliographystyle{acl_natbib}





\newpage

\appendix

\section{Limitations}
In this work, we identify the following limitations:

\begin{itemize}
    \item Due to limited computational resources, we could not explore a variety of pre-trained models beyond RoBERTa and LLaMA2. However, since most language models function similarly, we expect our methods and findings to be generally applicable.
    \item For LLaMA2 models, we only conduct few-shot prompting for downstream task evaluation due to resource constraints. Nonetheless, based on our observations with RoBERTa, fine-tuning on downstream tasks will likely show that CW2V and OFA are only marginally better than other approaches.
    \item Although we evaluated only five downstream tasks, we cannot confirm that our observations will apply to all types of tasks. This remains an area for future research.
    \item We show experiments on four languages—Hindi, German, Russian, and Tamil—due to limited computational resources. However, as we have chosen languages from different scripts, we expect our methods and findings to be generally applicable.
\end{itemize}

\section{Further Analysis}\label{converse}
\begin{theorem}: \textbf{All strongly good initializations are in the convex hull.}\label{theorem3}
   
Let $e_1^s, e_2^s, e_3 ^s, ..., e_n^s \in \mathbb{R}^d$ be the embeddings of words in $\mathcal{V}^s$. Let $y \in \mathbb{R}^d$. If $(h^Ty) \leq \sup_{k \in \mathcal{V}^s} (h^Te_k^s)$ for all $h \in \mathbb{R}^d$, then $y \in \mathcal{S}$, where $\mathcal{S}$ is the convex hull of the embeddings $e_1^s, e_2^s, e_3 ^s, ..., e_n^s$.
\end{theorem}

\begin{proof}
    We prove this using contradiction. Say, $y \notin \mathcal{S}$ and 
$(h^Ty) \leq \sup_{k \in \mathcal{V}^s} (h^Te_k^s)$ holds good for all $h \in \mathbb{R}^d$. Since, $\mathcal{S}$ is closed and convex and $y \notin \mathcal{S}$ , there exists a hyperplane $\mathbb{H}$ that strictly separates y from $\mathcal{S}$. This hyperplane defines a half space $\mathcal{H}$ containing $\mathcal{S}$. Note that $\mathcal{H}$ contains $\mathcal{S}$ and $y \notin \mathcal{H}$

\vspace{5pt}
 Let $\vec{b} \in \mathbb{R}^d$ be a point on the hyperplane $\mathbb{H}$. Let $\vec{n} \in \mathbb{R}^d$ denote the normal to the hyperplane $\mathbb{H}$. We choose $\vec{n}$ in such a way that any point $\vec{r} \in \mathcal{S}$ satisfies,

\begin{align*}
    (\vec{r} - \vec{b})^T \vec{n} \leq 0
\end{align*}

Thus, any embedding $e^s \in \{e_1^s, e_2^s, ... , e_n^s\}$ satisfies,
\begin{align}\label{eq:2}
    \left(e^s - b \right)^T \vec{n} \leq 0
\end{align}

and any point $\vec{q} \notin \mathcal{H}$ satisfies,

\begin{align*}
    (\vec{q} - \vec{b})^T \vec{n} \geq 0
\end{align*}

As $y \notin \mathcal{H}$,
\begin{align}\label{eq:3}
    (y - \vec{b})^T \vec{n} \geq 0
\end{align}

Equations \ref{eq:2} and \ref{eq:3} imply, 

\begin{align}
    \vec{n}^Te^s \leq \vec{n}^Ty \ \  \forall \ \ e^s  \in \{e_1^s, e_2^s, ... , e_n^s\}
\end{align}

Thus, $\sup_{k \in \mathcal{V}^s} (\vec{n}^Te_k^s) \leq (\vec{n}^Ty) $ which contradicts the statement that $ (h^Ty) \leq  \sup_{k \in \mathcal{V}^s} (h^Te_k^s) $ holds good for all $h \in \mathbb{R}^d$ as it fails for $h = \vec{n}$.

Thus, if $(h^Ty) \leq \sup_{k \in \mathcal{V}^s} (h^Te_k^s)$ for all $h \in \mathbb{R}^d$, then $y \in \mathcal{S}$, where $\mathcal{S}$ is the convex hull of the embeddings $e_1^s, e_2^s, e_3 ^s, ..., e_n^s$.

\end{proof}

Thus, from theorem \ref{theorem3} we can say that any `strongly good' initialization must lie in the convex hull of pre-expansion embeddings. But for an initialization to be considered `good', the output word must remain unchanged for prefixes formed by word sequences from the source vocabulary. This implies that the condition \ref{condition} only needs to be satisfied for a subset of $\mathbb{R}^d$, rather than for all $h \in \mathbb{R}^d$. Thus, it is not necessary that the converse of Theorem \ref{theorem2} to be true as we can have initializations which are `good' but not `strongly good'. However, we can say that if an initialization is `strongly good', embeddings must lie in the convex hull of pre-expansion embeddings.

\section{Effect on Initialisation on Model Output} \label{sec:a-table}
\input{table_llama_ouput}
Random initialization of new embeddings can result in a pre-trained language model assigning a probability of 1 to new words and can degrade domain adaptation performance \cite{hewitt2021initializing}. Figure \ref{fig:model_outputs} shows the outputs of expanded LLaMA2 models for an English sentence prompt. Random initialization of expanded tokens results in gibberish, while the other three methods produce outputs identical to the base LLaMA2 model, as they ensure embeddings lie within the convex hull of source embeddings.

\section{Fertility Score} \label{sec:a-fertility}
\input{table_fertility_score}
Table \ref{table_fertility_score} shows the fertility scores of the target tokenizer with respect to source tokenizer on 5 languages considered.




\section{Tokenizer Coverage}
\input{tokenizer_coverage}

Table \ref{tokenizer_coverage} shows the size of source vocabulary in experiments with RoBERTa and LLaMA2. As the new vocabulary is extended from LLaMA2, many subword
embeddings are directly copied when using LLaMA2 as
the source model. We employed a `fuzzy' search similar to FOCUS \cite{Dobler_2023} to identify the common tokens between the target tokenizer and the RoBERTa tokenizer. This led to a 38.5 \% coverage of tokens leading us to a source vocabulary of size 22K for experiments with RoBERTa.
 
\section{Do Multivariate and Univariate initializations reside in the hull?} \label{sec:a-multi}
In multivariate initialization, we sample from a multivariate Gaussian that considers correlations across dimensions, unlike the univariate distribution. When dealing with strongly correlated dimensions (positive or negative), a multivariate approach proves advantageous. By considering the correlations across dimensions, we can sample new embeddings that are positioned more effectively within the latent space of original embedding distribution. However, there is no straightforward method to determine if embedding sampled from either distribution lies within the hull. To ensure that multivariate initialization remains within the convex hull with a high confidence, we also scaled the covariance matrix by a factor of 1e-5. In contrast, unscaled univariate initialization was used as a baseline, aligning with previous studies \cite{liu2024ofa}. \cite{hewitt2021initializing} recommends employing multivariate initialization to incorporate noise. Notably, as illustrated in Figure 2, multivariate initialization significantly outperforms univariate initialization and closely approaches the performance of OFA in encoder-based models. However, a comprehensive theoretical analysis is required to determine if unscaled multivariate initialization has a higher likelihood of being within the convex hull compared to univariate initialization. This aspect is left for future research, given the empirical observation that univariate initializations typically exhibits lower performance compared to scaled multivariate initialization.

\section{Continued Pretraining Details}

All the expanded and initialized RoBERTa models are trained on the same hyperparameters used in OFA \cite{liu2024ofa}. Specifically, we employ the MLM objective with a standard mask rate of 15\%. We utilize the Adam optimizer \cite{kingma2017adammethodstochasticoptimization} with parameters $(\beta_1 = 0.9, \beta_2 = 0.999)$ and $\epsilon = 1 \times 10^{-6}$. The initial learning rate is set to $5 \times 10^{-5}$. The only deviation from our approach compared to OFA is the batch size, which is fixed at 32. Each batch consists of training samples concatenated up to the maximum sequence length of 512, randomly selected from all language-scripts described in Section \ref{data}. We continue to pretrain using the scripts adapted from HuggingFace\footnote{\url{https://github.com/huggingface/}}.

For LLaMa2, we used the standard LM objective with a context length of 2048 subwords. We used the Adam optimizer with linear warmup and decay where the peak learning rate was $5 \times 10^{-5}$ and warmup was done till 10\% of training steps. We trained for 1 epoch over our data saved checkpoints every 20\% of an epoch enabling us to study model behavior against increasing training data.

\section{Complete Results for Each Task and
Language}
Results for each task in all the languages across all the checkpoints is given in figures \ref{fig:enter-label1}, \ref{fig:enter-label2}, \ref{fig:enter-label3}, \ref{fig:enter-label4}, \ref{fig:enter-label5}, \ref{fig:enter-label6}, \ref{fig:enter-label7}

\begin{figure*}
    \centering
    \includegraphics[scale = 0.30]{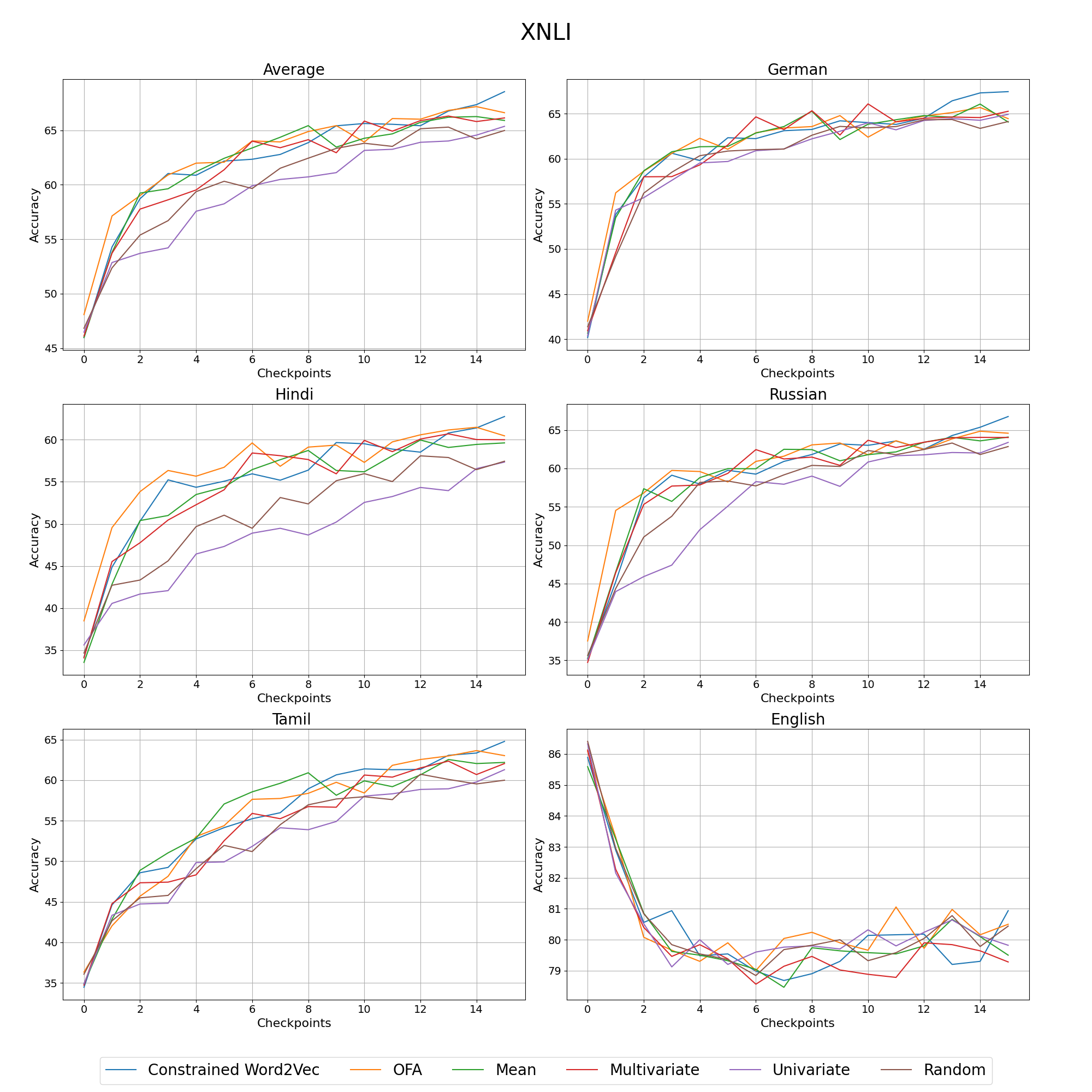}
    \caption{XNLI evaluation of expanded RoBERTa models}
    \label{fig:enter-label1}
\end{figure*}

\newpage
\begin{figure*}
    \includegraphics[scale = 0.30]{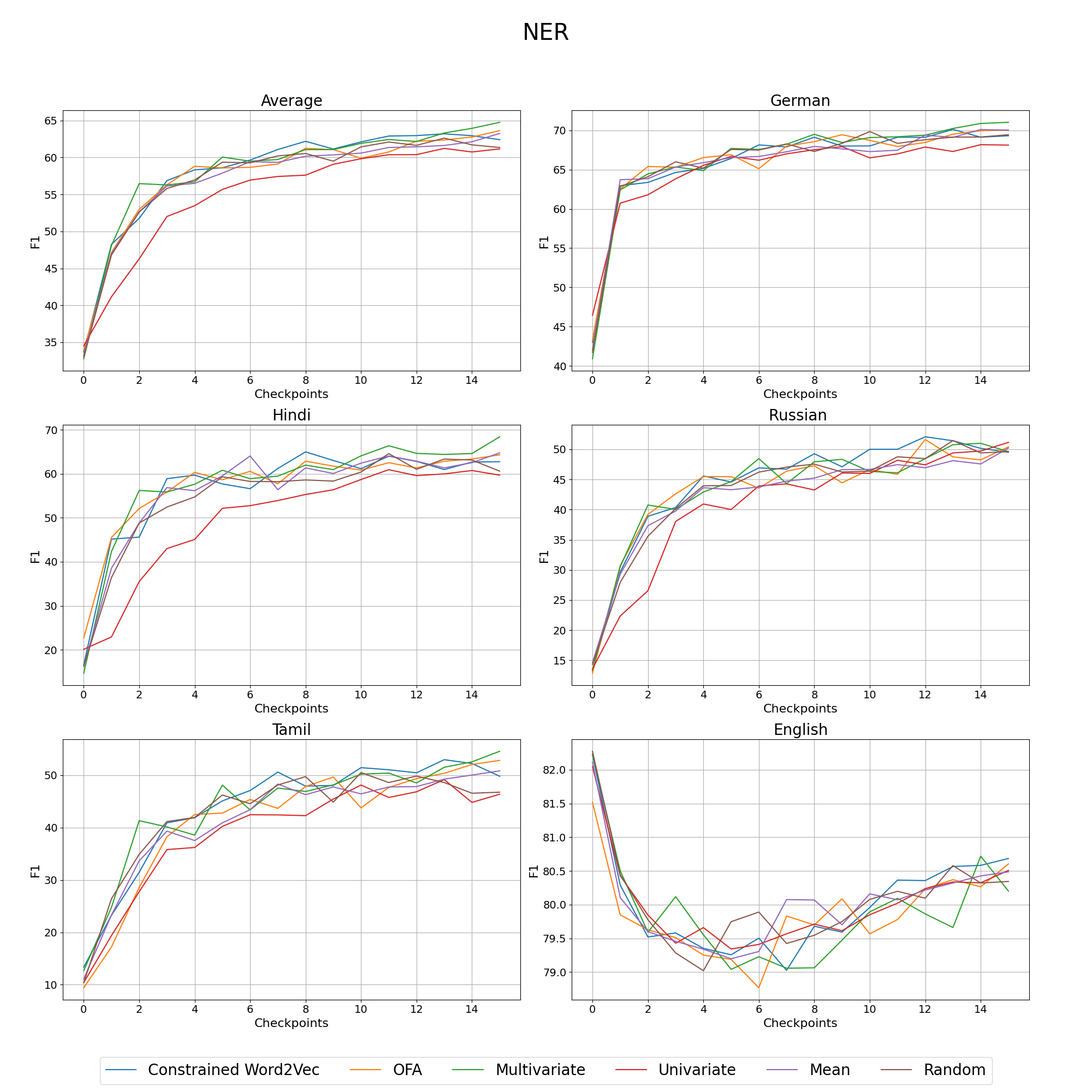}
    \caption{NER evaluation of expanded RoBERTa models}
    \label{fig:enter-label2}
\end{figure*}

\newpage
\begin{figure*}
    \includegraphics[scale = 0.30]{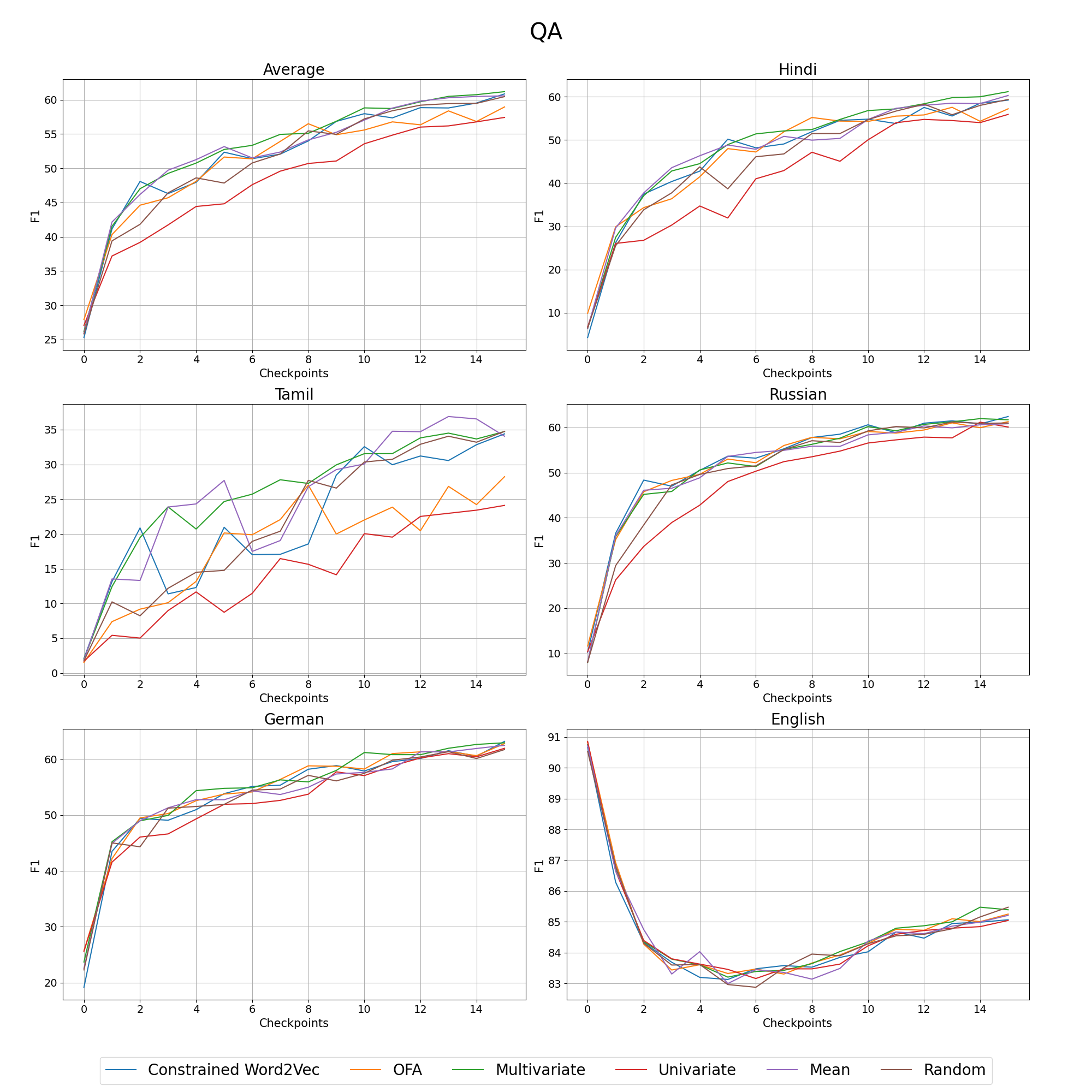}
    \caption{QA evaluation of expanded RoBERTa models}
    \label{fig:enter-label3}
\end{figure*}

\newpage

\begin{figure*}
    \centering
    \includegraphics[scale=0.21]{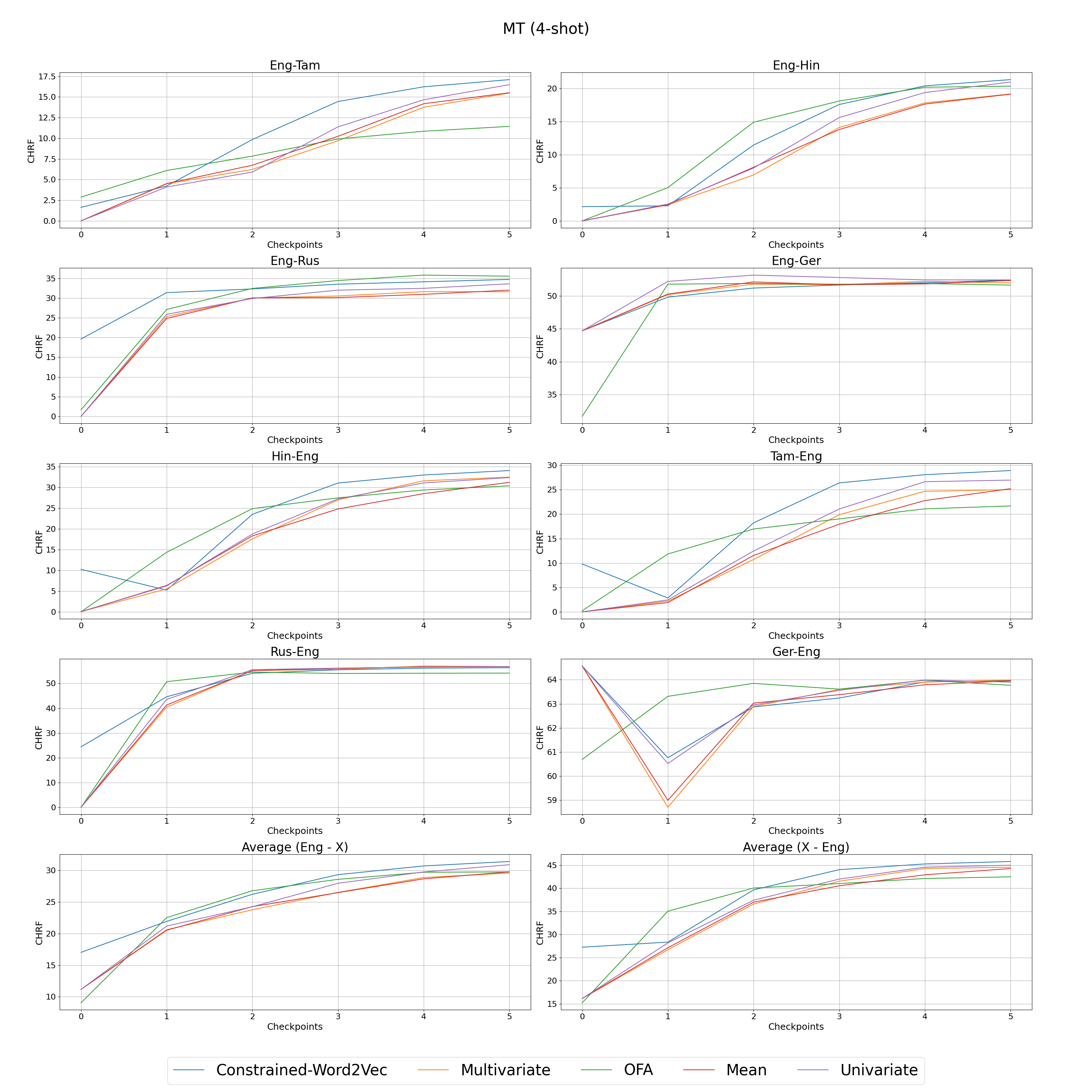}
    \caption{MT 4-shot evaluation of expanded LLaMA2 models}
    \label{fig:enter-label4}
\end{figure*}

\newpage

\begin{figure*}
    \centering
    \includegraphics[scale = 0.30]{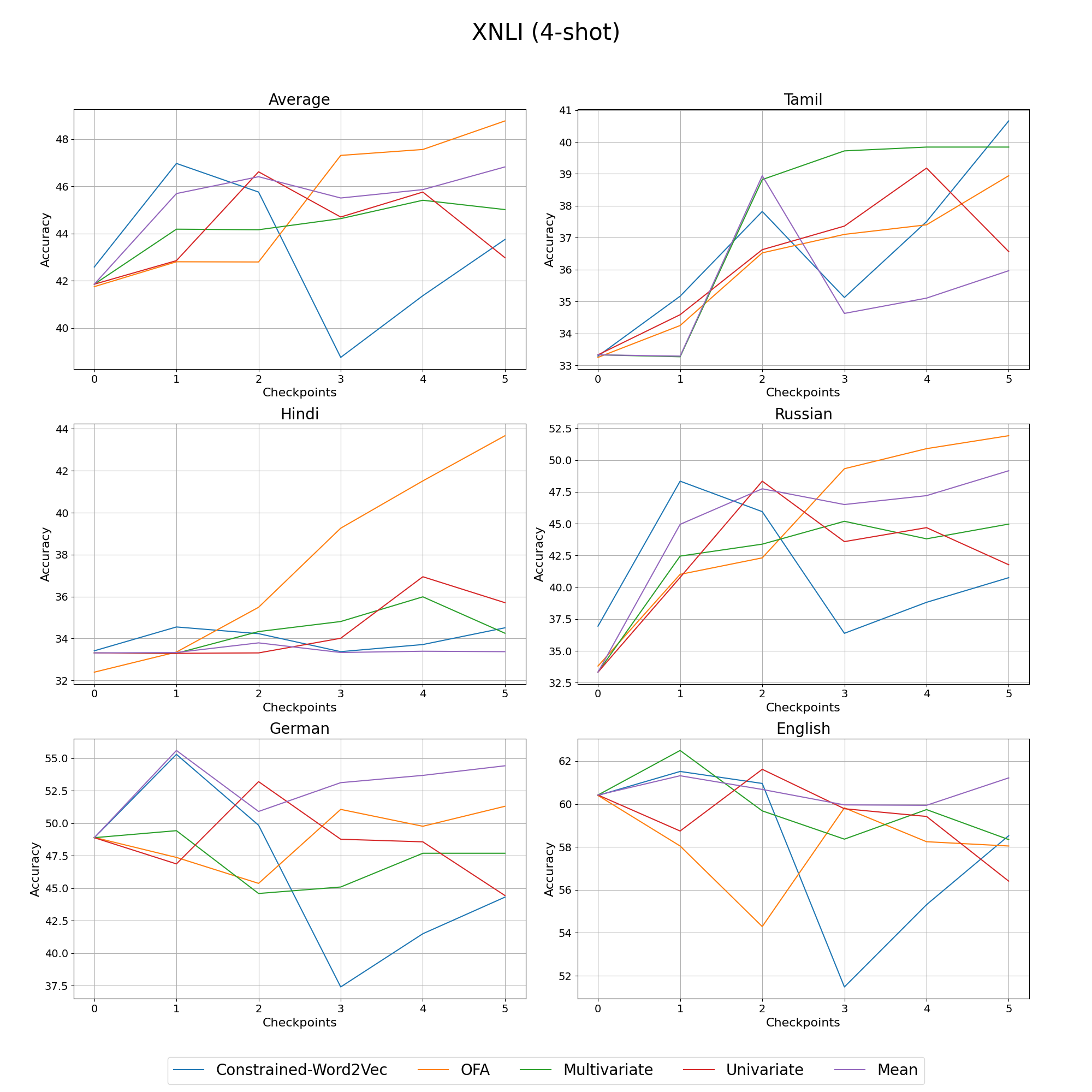}
    \caption{XNLI 4-shot evaluation of expanded LLaMA2 models}
    \label{fig:enter-label5}
\end{figure*}

\newpage
\begin{figure*}
    \centering
    \includegraphics[scale = 0.30]{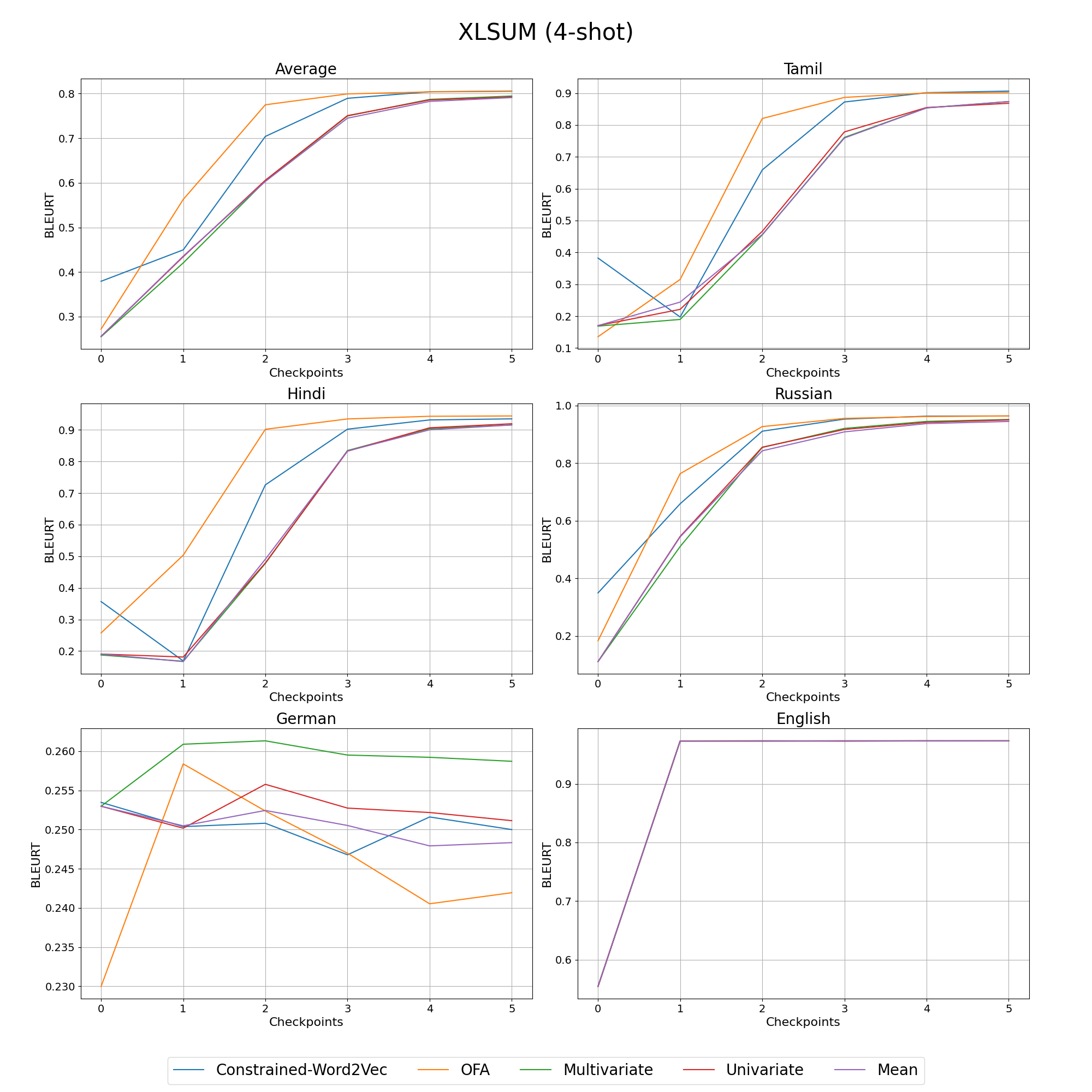}
    \caption{XLSUM 4-shot evaluation of expanded LLaMA2 models}
    \label{fig:enter-label6}
\end{figure*}

\newpage
\begin{figure*}
    \centering
    \includegraphics[scale = 0.30]{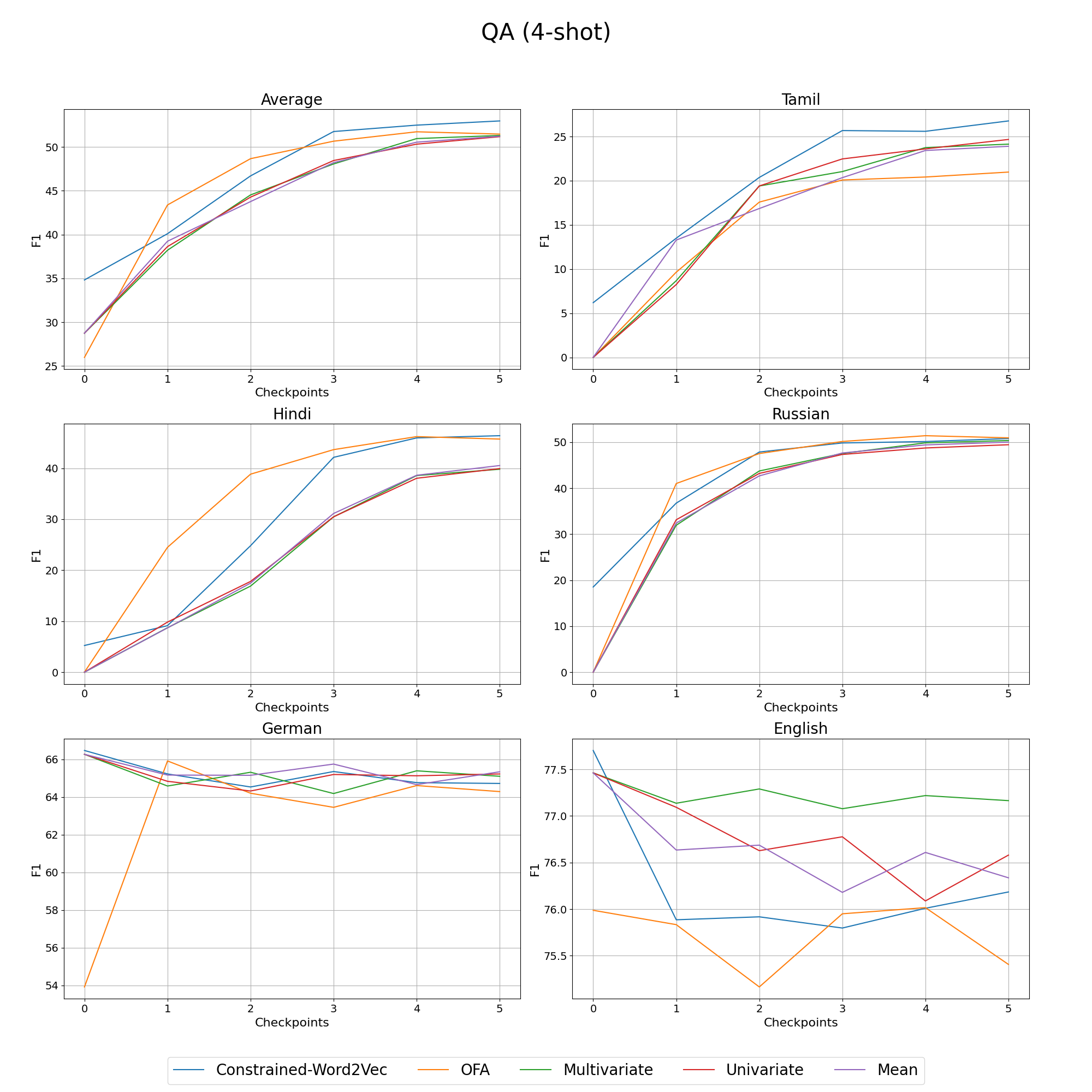}
    \caption{QA 4-shot evaluation of expanded LLaMA2 models}
    \label{fig:enter-label7}
\end{figure*}

\end{document}

%% file: table_task.tex
\begin{table}[h]
\centering
\small

\resizebox{\columnwidth}{!}{
\begin{tabular}{llcc}
\toprule
\textbf{Model} & \textbf{Task Category} & \textbf{Task} & \textbf{Metric} \\
\midrule
\multirow{3}{*}{\begin{tabular}[c]{@{}l@{}}\textbf{RoBERTa} \end{tabular}} & \begin{tabular}[c]{@{}l@{}}Sentence Classification\end{tabular} & XNLI &  Acc. \\
\cmidrule(lr){2-4}
 & Question Answering & QA & F1 \\
 \cmidrule(lr){2-4}
 & Token Classification & NER & F1 \\
 \midrule
\multirow{3}{*}{\begin{tabular}[c]{@{}l@{}}\textbf{LLaMA2} \end{tabular}} & \begin{tabular}[c]{@{}l@{}}Sentence Classification\end{tabular} & XNLI &  Acc. \\
\cmidrule(lr){2-4}
 & Machine Translation & FLORES & CHRF \\
\cmidrule(lr){2-4}
 & Question Answering & QA & F1 \\
 \cmidrule(lr){2-4}
 & Sentence Summarisation & XLSUM & BLEURT \\
\midrule
\end{tabular}
}
\caption{A summary of the tasks, datasets and metrics.}
\label{table:tasks}
\end{table}

%% file: table_llama_initialized.tex
\begin{table*}[t]
    \footnotesize
    \centering
    \begin{tabular}{lrrrrrr|rrrrrrrr}
    \toprule
          & \multicolumn{6}{c}{RoBERTa} & \multicolumn{8}{c}{LLaMA2} \\
          \cmidrule(lr){2-7} \cmidrule(lr){8-15}
          & \multicolumn{2}{c}{XNLI} & \multicolumn{2}{c}{NER} & \multicolumn{2}{c}{QA} & \multicolumn{2}{c}{MT} & \multicolumn{2}{c}{XNLI} & \multicolumn{2}{c}{QA} & \multicolumn{2}{c}{XLSUM}  \\
          \cmidrule(lr){2-3} \cmidrule(lr){4-5} \cmidrule(lr){6-7} \cmidrule(lr){8-9} \cmidrule(lr){10-11} \cmidrule(lr){12-13} \cmidrule(lr){14-15}
        ~ & en & avg & en & avg & en & avg & En-X & X-En & en & avg  & en & avg & en & avg\\ 
        \midrule
        CW2V & \textbf{86.0} & 36.0 & \textbf{82.2} & 21.5 &  \textbf{90.7}  & 9.0 & \textbf{17.0}  &  \textbf{27.3}  & 60.4  & \textbf{38.1}  & \textbf{77.7} & \textbf{35.8} & \textbf{0.6} & \textbf{0.4}\\ 
        OFA & 85.6 & \textbf{37.7} & 81.9 & \textbf{21.7} & 90.6  & \textbf{12.0} & 11.2  & 16.2 & 60.4 & 37.1 & 76.0 & 26.0 & 0.6 & 0.3\\ 
        Multivariate & 85.7 & 35.7 & 81.8 & 18.3 & 90.4 & 9.5 & 11.1 & 16.1 & 60.4 & 37.2 & 77.5 & 28.7 & 0.5 & 0.2\\ 
        Univariate & 85.6 & 36.6 & 82.0 & 22.0 & 90.7 & 10.3 & 11.1 & 16.0 & 60.4 & 37.2 & 77.4 & 28.7 & 0.5 & 0.3\\ 
        Mean & 85.5 & 36.0 & 81.5 & 20.3 & 90.5 &  8.8 & 11.1 & 16.2 & \textbf{60.5} & 37.2 & 77.4 & 28.7 & 0.5 & 0.3\\
        Random & 85.8 & 35.9  & 81.6 &  21.0 & 90.3 & 9.6 & 0.0 & 0.0 & 33.3 & 33.3 & 0.0 & 0.0 & 0.0 & 0.0 \\
               \bottomrule
    \end{tabular}
    \caption{Performance of the expanded RoBERTa and LLaMA2 models initialized with Constrained Word2Vec and baselines on downstream tasks across 5 languages. }
    \label{tab:results}
\end{table*}

%% file: table_llama_ouput.tex

\begin{figure}[h]
    \centering
    \includegraphics[scale = 0.4]{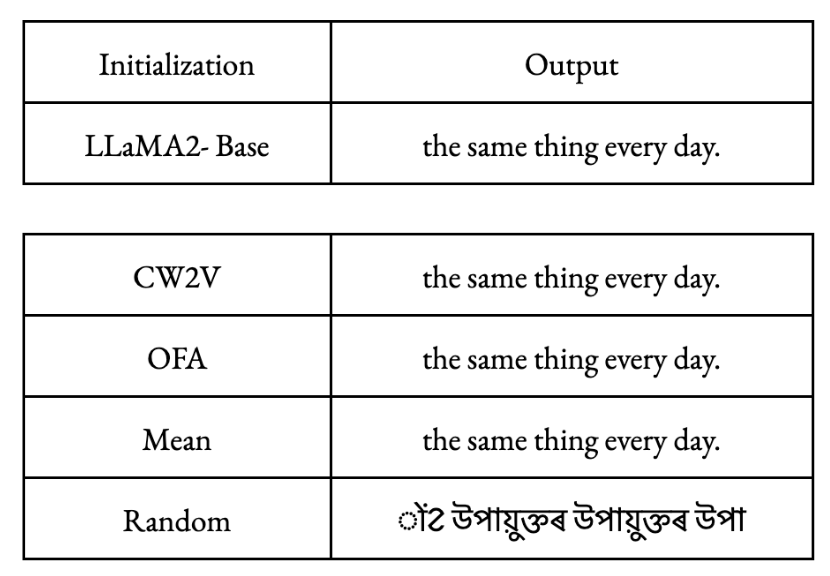}
    \caption{Expanded LLaMA2 Model Outputs for the Prompt : \textit{``I don't want to eat" for various initializations.}}
     \label{fig:model_outputs}
\end{figure}

%% file: table_fertility_score.tex
\begin{table}[!htbp]
    \centering
    \resizebox{\columnwidth}{!}{
    \begin{tabular}{|l|l|l|l|l|l|}
    \hline
        Fertility Score & English & Hindi & Tamil & Russian & German \\ \hline
        LLaMA2 Tokenizer & 2.89 & 7.47 & 12.66 & 4.25 & 3.88 \\ \hline
        RoBERTa Tokenizer & 2.87 & 10.85 & 28.80 & 9.89 & 4.42 \\ \hline
        Extended Tokenizer & \textbf{2.87} & \textbf{2.83} & \textbf{2.83} & \textbf{3.74} & \textbf{3.88} \\ \hline
    \end{tabular}
    }
    \caption{Fertility scores for the source and the extended tokenizers on all the languages}
    \label{table_fertility_score}
\end{table}

%% file: tokenizer_coverage.tex
\begin{table} [h] 
    \setlength{\belowcaptionskip}{-0.4cm}
    \footnotesize
    \centering
    \setlength{\tabcolsep}{1.2mm}{}
    \begin{tabular}{ll|l|l}
        \toprule
        & \multicolumn{3}{c}{Target Tokenizer}  \\
        \cmidrule(lr){2-4} 
        & Copied Tokens & Initialized Tokens & Coverage\\
        \midrule
        RoBERTa & 22K & 35K &  38.5 \% \\
        LLaMA2 & 32K& 25K &  56.14 \%\\
       
        
       \bottomrule
    \end{tabular}
    \caption{: The number of subwords being initialized
by copying from the original embeddings from RoBERTa's and LLaMA's tokenizers.  }
    \label{tokenizer_coverage}
\end{table}